\documentclass[10pt]{amsart}

\usepackage{amsmath}
\usepackage{mathtools}
\usepackage{dsfont}
\usepackage{mathabx}

\usepackage[usenames, dvipsnames]{color}
\usepackage[colorlinks=true,
        raiselinks=true,
        linkcolor=MidnightBlue,
        citecolor=Brown,
        urlcolor=ForestGreen,
        pdfauthor={},
        pdftitle={},
        pdfkeywords={},
        pdfsubject={},
        plainpages=false]{hyperref}
\usepackage[T1]{fontenc}

\usepackage[norelsize, ruled, vlined]{algorithm2e}

\usepackage{enumitem}
\setlist{noitemsep, topsep=0cm}

\usepackage{tikz}
\usepackage{float}
\usetikzlibrary{fadings}
\usetikzlibrary{shapes.misc}
\tikzset{cross/.style={cross out, draw=black, minimum size=2*(#1-\pgflinewidth), inner sep=0pt, outer sep=0pt},cross/.default={1pt}}
\usetikzlibrary{shapes,arrows}
\usetikzlibrary{positioning}
\usetikzlibrary{plotmarks}

\allowdisplaybreaks

\newtheorem{theorem}{Theorem}[section]
\newtheorem{proposition}[theorem]{Proposition}
\newtheorem{lemma}[theorem]{Lemma}
\newtheorem{corollary}[theorem]{Corollary}
\newtheorem{definition}[theorem]{Definition}

\theoremstyle{remark}
\newtheorem{remark}[theorem]{Remark}
\newtheorem{assumption}[theorem]{Assumption}
\theoremstyle{claim}

\theoremstyle{definition}

\newcommand{\norm}[2]{\left\lVert#1\right\rVert_{#2}}
\newcommand{\abs}[1]{\left\lvert #1 \right\rvert}
\newcommand{\pmat}[3]{\begin{pmatrix} #1 & #2 &\cdots & #3 \end{pmatrix}}

\newcommand{\dict}[1]{d (#1)}
\newcommand{\inprod}[2]{\left\langle#1, \; #2\right\rangle}
\newcommand{\inpnorm}[1]{\norm{#1}{}}

\newcommand{\dimension}{n}

\newcommand{\summ}[2]{\sum\limits_{#1}^{#2}}

\DeclareSymbolFont{symbolsC}{U}{pxsyc}{m}{n}

\DeclareMathOperator{\trace}{tr}

\DeclareMathOperator{\EE}{\mathsf{E}}
\DeclareMathOperator{\PP}{\mathsf{P}}

\DeclareMathOperator*{\minimize}{minimize}
\DeclareMathOperator*{\sbjto}{subject \; to}

\DeclareMathOperator*{\argmin}{argmin}
\DeclareMathOperator*{\argmax}{argmax}

\DeclarePairedDelimiterX\set[1]\lbrace\rbrace{#1}

\newcommand{\define}{\coloneqq}

\newcommand{\R}[1]{\mathbb{R}^{#1}}
\newcommand{\opt}{^\ast}
\newcommand{\transp}{^\top}
\newcommand{\dsize}{K}
\newcommand{\dictionary}{D}
\newcommand{\unitdictionaryset}{\mathcal{\dictionary}}
\newcommand{\nbhood}[2]{B(#1, #2)}
\newcommand{\closednbhood}[2]{B[#1, #2]}

\renewcommand{\geq}{\geqslant}

\renewcommand{\leq}{\leqslant}

\newcommand{\hilbert}{\mathbb{H}_{\dimension}}
\newcommand{\rv}{X}
\newcommand{\repvec}{f}
\newcommand{\error}{\epsilon}
\newcommand{\regulizer}{\delta}

\newcommand{\horizon}{T}

\newcommand{\sample}[1]{x_{#1}}
\newcommand{\samplevec}{x}
\newcommand{\dualvar}[1]{\eta_{#1}}
\newcommand{\hvar}[1]{h_{#1}}
\newcommand{\separator}[1]{\lambda_{#1}}
\newcommand{\separatorset}{\Lambda_{\regulizer} (\dictionary, \samplevec, \error)}
\renewcommand{\separatorset}[2]{\Lambda (#1, #2)}
\newcommand{\Depsdelta}{(\dictionary, \error, \regulizer)}

\newcommand{\encodermap}[2]{F_{#2} (#1)}
\newcommand{\codes}{\encodermap{\dictionary}{\samplevec}}
\newcommand{\unitencoder}[2]{H_{#2} (#1)}

\newcommand{\encodedcost}[3]{C(#1, #2, #3)}
\newcommand{\samplecost}[2]{C_{#2} (#1)}
\newcommand{\cost}{c}
\newcommand{\costatomset}{V_{\cost}}
\newcommand{\horder}{p}

\newcommand{\proj}{\pi}

\title[Dictionary Learning with Almost Sure Error Constraints]{Dictionary Learning with Almost Sure Error Constraints}

\author[Mohammed Rayyan Sheriff, Debasish Chatterjee]{MOHAMMED RAYYAN SHERIFF, DEBASISH CHATTERJEE}
\email{mohammedrayyan@sc.iitb.ac.in, dchatter@iitb.ac.in}
\address{Systems \& Control Engineering, Indian Institute of Technology Bombay, Powai, Mumbai~400076, India.}

\begin{document}
\maketitle 

\begin{abstract}
A dictionary is a database of standard vectors, so that other vectors or signals are expressed as linear combinations of dictionary vectors, and the task of learning a dictionary for a given data is to find a good dictionary so that the representation of data points has desirable features. Dictionary learning and the related matrix factorization methods have gained significant prominence recently due to their applications in Wide variety of fields like machine learning, signal processing, statistics, etc. In this article, we study the dictionary learning problem for achieving desirable features in the representation of a given data with almost sure recovery constraints. We impose the constraint that every sample is reconstructed properly to within a predefined threshold. This problem formulation is more challenging than the conventional dictionary learning, which is done by minimizing a regularised cost function. We make use of the duality results for linear inverse problems to obtain an equivalent reformulation in the form of a convex-concave min-max problem. The resulting min-max problem is then solved using gradient descent-ascent like algorithms.
\end{abstract}

\section{Introduction}
Signals have almost always been expressed as a linear combination of a standard database / collection of vectors. For instance, audio signals have historically been studied by expressing them as linear combination of Fourier basis, wavelets \cite{mallat1999wavelet}, \cite{daubechies1992ten} etc. It turns out that expressing signals in a well chosen basis helps in the study of the underlying characteristics of the signal than it is in its natural representation. In fact, it is often the case that signals in their natural representation are elements of a very high dimensional vector space even though there exists some low dimensional characteristics that could be exploited for its effective representation.

Almost all of natural signals are driven / outcomes of processes that are inherently low dimensional. Therefore, there is a lot of redundancy in data that is typically encountered in practise. Moreover, in certain applications \cite{elhamifar2013sparse},\cite{soltanolkotabi2014robust} data is concentrated around low dimensional subspaces or some such sort of clusters. Since every standard basis that is used in classical signal processing techniques like a Fourier basis, wavelets, DCT etc., is orthonormal, it turns out that most of the basis vectors do not contribute much due to the very nature of orthonormality. Therefore, representation using such bases does not maximally exploit the redundancy in the data. Later \emph{Frame theoretic} ideas \cite{daubechies1986painless}, \cite{casazza2012finite} have shown that by relaxing the orthonormality constraint and allowing the standard database to have more vectors than the effective dimension of the data, it allows us to exploit the redundancy in the data better than the standard orthonormal bases. However, it is a natural question to ask; what is the best database of vectors one has to use so that the given data is optimally represented or analysed. The recent advent of \emph{Dictionary Learning} techniques \cite{aharon2006k} \cite{olshausen1996emergence} \cite{olshausen1997sparse} \cite{mairal2010online} \cite{mairal2011task} is an attempt to accomplish this task. In these techniques, a collection of vectors called \emph{atoms} that constitute a database referred to as the \emph{dictionary} is learned from the data and for the data with a desired objective. It has been successfully shown \cite{aharon2006k} \cite{wright2010sparse} \cite{fevotte2009nonnegative} \cite{mairal2007sparse} that learning a dictionary that is adapted to the data often outperforms the classical techniques by a considerable margin in a plethora of signal processing applications. For a brief overview of dictionary learning techniques and their application, see \cite{tosic2011dictionary}.

One of primary feature that is central to the success of modern day signal processing techniques is sparsity. Sparsity based techniques have been successfully implemented in tasks like signal compression \cite{donoho2006compressed} \cite{candes2008introduction} \cite{candes2006near}, denoising \cite{dong2011sparsity} \cite{fang2012sparsity}, clustering \cite{ramirez2010classification} etc. In particular, the inception and success of Compressed Sensing \cite{donoho2006compressed} \cite{candes2008introduction} is noteworthy. Even though signals might not be sparse in their natural representation, they can be approximated reasonably well by a sparse linear combination of the atoms in some dictionary. For instance, natural images admit a reasonably approximate sparse representation in \( 2 \)-d Discrete Cosine Transform (DCT) basis, typically with only less than five percent of the coefficients being non-zero. Given the fact that such hidden sparsity is common in signals encountered in the current age of big data, the use of sparsity based signal processing techniques has become more compelling. Therefore, for a given data, it is desirable to learn a dictionary that allows the possibility of sparse representation of the signals without losing much information.

Let \( \samplevec \) be the signal that admits a reasonably approximate sparse representation \( \repvec_{\samplevec} \) (that is unknown and needs to be computed) in a dictionary \( \dictionary = \pmat{\dict{1}}{\dict{2}}{\dict{\dsize}} \). The sparse representation is computed by solving the following convex optimization problem
\begin{equation}
\label{eq:Basis-Pursuit-Denoising}
\repvec_{\error} \in 
\begin{cases}
\begin{aligned}
		& \argmin_{ \repvec } && \norm{ \repvec }{1} \\
		& \sbjto				&& 
		\norm{\samplevec - \dictionary \repvec}{2} \leq \error ,
\end{aligned}
\end{cases}
\end{equation}
where \( \error \) is a positive real number that signifies the permissible error in approximation. Sometimes, in applications like denoising etc., \( \error \) is the bound on the noise.  Ideally one would want to minimize the \( \ell_0 \)-pseudo norm \( \norm{\repvec}{0} \define \lvert \{ i : \repvec_i \neq 0 \} \rvert \), however, it leads to intractability for large scale problems. Fortunately, its convex relaxation \eqref{eq:Basis-Pursuit-Denoising} works reasonably well in most of the practical setting. Given a data set \( (\samplevec_t)_t \), we desire to use a ``good'' dictionary \( \dictionary \) that offers better sparse representation of the data. The primary goal of this article is to learn such a dictionary from a given data set by solving the optimization problem
\begin{equation}
\label{eq:DL-SCP}
\begin{cases}
\begin{aligned}
		& \minimize_{ \dictionary , \; (\repvec_t)_t } && \frac{1}{T} \sum\limits_{t = 1}^T \norm{ \repvec_t }{1} \\
		& \sbjto				&& 
		\begin{cases}
		    \dictionary \in \unitdictionaryset , \; \repvec_t \in \R{\dsize} \\
			\norm{\samplevec_t - \dictionary \repvec_t}{2} \leq \error_t \quad \text{for every \( t = 1, 2, \ldots, T \)} ,
		\end{cases}
\end{aligned}
\end{cases}
\end{equation}
where  \( \dimension \) and \( \dsize \) are some positive integers and the set \( \unitdictionaryset \) is some compact convex subset of \( \R{\dimension \times \dsize} \). The set \( \unitdictionaryset \) is typically chosen to be such that every dictionary vector is of at most unit length.

Alternatively, for a given dictionary \( \dictionary \), the sparse representation can also be obtained by solving the following regularised formulation
\begin{equation}
\label{eq:l1-regularized-sparse-coding}
\repvec'_{\gamma} \in \; \argmin_{\repvec} \; \Big( \norm{ f }{1} \; + \; \gamma \norm{x - \dictionary f}{2}^2 \Big) ,
\end{equation}
where \( \gamma > 0 \) is a regularization parameter. Evidently, the objective function is a weighted cost of sparsity inducing \( \ell_1 \)-penalty and the error in representation. This trade off is controlled by the regularization parameter \( \gamma \). For large values of \( \gamma \), the representation is more accurate but less sparse and vice versa. For a given data \( (\samplevec_t)_t \), assuming that one has the knowledge of a good value of \( \gamma \), dictionary learning is done conventionally by solving the following optimization problem
\begin{equation}
	\label{eq:l1-regularized-DL}
	\minimize_{(\repvec_t)_t, \; \dictionary \; \in \; \unitdictionaryset } \quad \frac{1}{T} \sum\limits_{t = 1}^T  \Big( \norm{ \repvec_t }{1} \; + \; \gamma \norm{\samplevec_t - \dictionary \repvec_t}{2}^2  \Big).
\end{equation}

Whenever \eqref{eq:Basis-Pursuit-Denoising} is strictly feasible, it is a known fact that the problems \eqref{eq:Basis-Pursuit-Denoising} and \eqref{eq:l1-regularized-sparse-coding} are equivalent. For a given value of \( \error > 0 \) (\( \gamma > 0 \)) there exists some \( \gamma (\error) > 0 \) (\( \error(\gamma) > 0 \)) such that the problems \eqref{eq:Basis-Pursuit-Denoising} and \eqref{eq:l1-regularized-sparse-coding} with parameters \( \error \) and \( \gamma (\error) \) (\( \error(\gamma) \) and \( \gamma \)) respectively, admit identical optimal solutions. This implies that, given a dictionary one could chose either \eqref{eq:Basis-Pursuit-Denoising} or \eqref{eq:l1-regularized-sparse-coding} on convenience of implementation to obtain the sparse representation of the data. However, the corresponding dictionary learning problems \eqref{eq:DL-SCP} and \eqref{eq:l1-regularized-DL} need not be equivalent. Almost all of notable recent work on dictionary learning has been concentrated towards solving
\eqref{eq:l1-regularized-DL}. On the contrary, there is little work done to solve the dictionary learning problem \eqref{eq:DL-SCP} in a meaningful manner. 

The primary goal of this article is to solve the dictionary learning problem \eqref{eq:DL-SCP} in situations where the knowledge of good values of \( (\error_t)_t \) are known. In fact, in many image processing applications like denoising, inpainting etc., where the image is corrupted by some noise, it is often the case that good statistical information of the noise is easily available. Since the values of \( (\error_t)_t \) that are to be used for such applications depend on the noise characteristics, good estimates of their values are available beforehand. Therefore, in situations like these, considering the dictionary learning problem in the formulation \eqref{eq:DL-SCP} is natural. Moreover, an alternate perspective to look at \eqref{eq:DL-SCP} is that we are putting a hard constraint on the permissible error, and then optimizing for sparsity. This is advantageous because, such a formulation provides the user the possibility to specify the maximum permissible error limit.

There is a tremendous body of work already done to address the problem \eqref{eq:l1-regularized-DL}, we want to highlight that such techniques can't be applied directly to solve \eqref{eq:DL-SCP}. In the sparse coding problem \eqref{eq:l1-regularized-sparse-coding}, the regularization parameter \( \gamma \) controls the tradeoff between sparsity and the error terms. For a given value of \( \gamma \), if the optimal solution to \eqref{eq:l1-regularized-sparse-coding} is \( \repvec'_{\gamma} \), one does not know the value of the error \( \norm{\samplevec - \dictionary \repvec'_{\gamma}}{2} \) incurred before actually solving the problem \eqref{eq:l1-regularized-sparse-coding}. If the specified error bound is \( \error \), there is no way to chose \( \gamma \) apriori such that the error bound \( \norm{\samplevec - \dictionary \repvec'_{\gamma}}{2} \leq \error \) is ensured. This is due to the fact that even though the problems \eqref{eq:Basis-Pursuit-Denoising} and \eqref{eq:l1-regularized-sparse-coding} are equivalent, the relation \( \error \longmapsto \gamma (\error) \) is not straight forward and unknown beforehand. In fact such a relation depends on the point \( \samplevec \) and the dictionary \( \dictionary \). Furthermore, if the sparse representations \( (\repvec_t)_t \) are obtained by solving \eqref{eq:l1-regularized-sparse-coding} by using the same value of \( \gamma \) for each \( t \), it is very likely that the error constraint \( \norm{\samplevec_t - \dictionary \repvec_t}{2} \leq \error_t \) is not satisfied for all \( t \). 

For the sake of argument, let us ignore the error constraint and simply consider the task of learning a dictionary by solving the following problem instead of \eqref{eq:l1-regularized-DL}.
\begin{equation}
\label{eq:sample-regularized-DL}
	\minimize_{(\repvec_t)_t, \; \dictionary \; \in \; \unitdictionaryset } \quad \frac{1}{T} \sum\limits_{t = 1}^T  \Big( \norm{ \repvec_t }{1} \; + \; \gamma_t \norm{\samplevec_t - \dictionary \repvec_t}{2}^2  \Big).
\end{equation}
Allowing independent regularizer for each \( t \) gives the user, more freedom and control which is apparent in the formulation \eqref{eq:DL-SCP}. It is clear that the performance of the dictionary learned this way for data analysis applications is critical to the regularisation parameters \( (\gamma_t)_t \) used to solve \eqref{eq:sample-regularized-DL}. Therefore, it is of high importance to know the right value of the regularization parameters that are best for the given data beforehand. One of the main challenges in learning a dictionary via this formulation or \eqref{eq:l1-regularized-DL} is that, apriori we do not know such regualarizer values. Typically, they are learned through cross validation techniques by solving multiple versions of the problem \eqref{eq:l1-regularized-DL} for the data with different values of the regularizers. This means that when the data set is large, which is typical of the modern times, one has to solve the dictionary learning problem multiple times, thereby increasing the demand on computational requirements.

It is to be noted that none of the problems \eqref{eq:DL-SCP}, \eqref{eq:l1-regularized-DL} and \eqref{eq:sample-regularized-DL} is jointly convex in arguments \( (\repvec_t)_t \) and \( \dictionary \). However, all of them are convex with respect to each argument given that the other is held fixed. Due to this reason, dictionary learning problems \eqref{eq:l1-regularized-DL} and \eqref{eq:sample-regularized-DL} are solved by alternating the minimization over \( (\repvec_t)_t \) and \( \dictionary \) iteratively. Therefore, the dictionary is updated from \( \dictionary \longmapsto \dictionary' \) in the following manner
\begin{equation}
\label{eq:DL-alternating-problems}
\begin{cases}
\begin{aligned}
	 \repvec'_t & \in \; \argmin_{\repvec_t} && \Big( \norm{ \repvec_t }{1} \; + \; \gamma_t \norm{\samplevec_t - \dictionary \repvec_t}{2}^2 \Big) \quad \text{for every \( t \), and} \\
     \dictionary' & \in \; \argmin_{\dictionary \; \in \; \unitdictionaryset} && \frac{1}{T} \sum\limits_{t = 1}^T \norm{\samplevec_t - \dictionary \repvec'_t}{2}^2  .
\end{aligned}
\end{cases}
\end{equation}
It is immediate that the dictionary update step is a Quadratic program which can be solved efficiently. In fact, based on co-ordinate descent methods to solve this optimization problem, in turns out that the alternating minimization to learn a good dictionary can be done online as in \cite{mairal2010online}. 

On the contrary, the dictionary learning problem \eqref{eq:DL-SCP} readily does not admit such an alternating minimization strategy. Once we fix \( (\repvec_t)_t\), the cost function in \eqref{eq:DL-SCP} remains constant for every dictionary such that \( \inpnorm{\samplevec_t - \dictionary \repvec_t} \leq \error_t \) for every \( t \). Therefore, there is no obvious way to update the dictionary variable. This is perhaps one of the primary reason why dictionary learning problem \eqref{eq:DL-SCP} hasn't received equal attention as \eqref{eq:l1-regularized-DL} in the mainstream. In practise, where solving the dictionary learning problem \eqref{eq:DL-SCP} is necessary, it is still done by a slightly different but similar alternating minimization technique
\begin{equation}
\label{eq:DL-error-constrained-alternating-problems}
\begin{cases}
\begin{aligned}
\repvec'_t & \in 
\begin{cases}
 \begin{aligned}
  & \minimize_{ \repvec_t }  &&  \norm{ \repvec_t }{1} \\
  & \sbjto  &&  \norm{\samplevec_t - \dictionary \repvec_t}{2} \leq \error_t
 \end{aligned}
\end{cases}
\text{for every \( t \), and} \\
\dictionary' & \in \; \minimize_{\dictionary \; \in \; \unitdictionaryset} \quad \frac{1}{T} \sum\limits_{t = 1}^T \norm{\samplevec_t - \dictionary \repvec'_t}{2}^2  .
\end{aligned}
\end{cases} 
\end{equation}
However, by replacing the convex problem of minimization over \( (\repvec_t)_t \) in \eqref{eq:DL-SCP} with its Lagrange dual, we get the following min-max-min problem equivalent to \eqref{eq:DL-SCP}
\begin{equation}
\label{eq:DL-error-constrained-with-coding-problem-dual}
\begin{cases}
\begin{aligned}
& \min_{ \dictionary } \max_{(\gamma_t)_t} \min_{(\repvec_t)_t} && \frac{1}{T} \sum\limits_{t = 1}^T  \Big( \norm{ \repvec_t }{1} \; + \; \gamma_t \big( \norm{\samplevec_t - \dictionary \repvec_t}{2}^2 - \error_t^2 \big) \Big) \\
& \sbjto && 
\begin{cases}
\dictionary \in \unitdictionaryset \\
\gamma_t > 0 \text{ for every } t .
\end{cases}
\end{aligned}
\end{cases}
\end{equation}
It is clear from \eqref{eq:DL-error-constrained-alternating-problems} and \eqref{eq:DL-error-constrained-with-coding-problem-dual} that the dictionary update \eqref{eq:DL-error-constrained-alternating-problems} completely disregards the maximization over the dual variables \( (\gamma_t)_t \), and treats them instead as constants.\footnote{The choice of notation \( \gamma_t \) for dual variables is intentional. In fact the equivalence between \eqref{eq:Basis-Pursuit-Denoising} and \eqref{eq:l1-regularized-sparse-coding} is immediate from the fact that for a given \( \error > 0 \), \( \gamma (\error) \) is precisely the optimal value of the dual variable in the Lagrange dual problem to \eqref{eq:Basis-Pursuit-Denoising}.} Therefore, there is no mathematical justification whatsoever on why dictionaries updated via \eqref{eq:DL-error-constrained-alternating-problems} should eventually be optimal solutions to \eqref{eq:DL-SCP}.
 
The main difficulty in solving \eqref{eq:DL-SCP} is that the dictionary variable does not appear directly in the objective function. It affects the feasibilty of a candidate \( \repvec_t \), and thereby affecting the cost indirectly. This makes it impossible to solve \eqref{eq:DL-SCP} through obvious alternating minimization techniques. We have considered a slightly more general problem formulation to \eqref{eq:DL-SCP} in this article, and have solved it. We go about doing this by replacing \eqref{eq:Basis-Pursuit-Denoising} with an equivalent convex-concave min-max formulation provided in \cite{sheriff2019LIP} that pushes the dictionary variable to the cost function. This allows us to update the dictionary in a meaningful manner that minimizes the objective function of \eqref{eq:DL-SCP} in each iteration. Learning a dictionary to solve \eqref{eq:DL-SCP} by methods provided in this article is not only mathematically justified but outperforms the conventional techniques like \eqref{eq:DL-error-constrained-alternating-problems} significantly. 

The Chapter unfolds as follows, in Section \ref{section:problem-statement-and-main-result} we formally define the dictionary learning problem and its associated encoding problem. In Section \ref{section:main-results-and-algo}, we provide our main results and algorithms. In Section \ref{section:proofs}  we provide mathematical proofs for the results in Section \ref{section:problem-statement-and-main-result}.

%\textbfasso{Addciated en inproblemfocoding  on ill posedness.} if wm and its e were to apply the alternating minimization technique directly on the original problem \eqref{eq:DL-conventional} by alternating between minimization of variables \( \dictionary \) and \( (\repvec_t)_t \). We immediately see that 

%%%%% Other advantages of our method if any others

%%% Contribution

%%% Article structure

%%% Notations

%%%%%%%%%%%%%%%%%%%%%%%%%%%%%%%%%%%%%%%%%%%%%%%%%%%%%%%%%%%%%%%%%%%%%%%%%%%%%%%%%%%%%%%%%%%%%%%%%%%%

\section{The dictionary learning problem and its solution}
\label{section:problem-statement-and-main-result}
Let \( \dimension \) be a positive integer, \( \hilbert \) be an \( \dimension \)-dimensional Hilbert space equipped with an innerproduct \( \inprod{\cdot}{\cdot} \) and its associated norm \( \inpnorm{\cdot} \). For every \( x \in \hilbert \) and \( r > 0 \), let \( \nbhood{x}{r} \define \{ y \in \hilbert : \inpnorm{\samplevec - y} < \error \} \) and let \( \closednbhood{x}{r} \define \{ y \in \hilbert : \inpnorm{\samplevec - y} \leq \error \} \).

Every vector \( \samplevec \in \hilbert \) is \emph{encoded} as a vector \( \repvec ( \samplevec) \) in \( \R{\dsize} \) via the \emph{encoder} map \( \repvec : \hilbert \longrightarrow \R{\dsize} \). The \emph{reconstruction} of the encoded samples from the codes \( \repvec (\samplevec) \) is done by taking the linear combination \( \sum\limits_{i = 1}^{\dsize} \repvec_i (\samplevec) \dict{i} \) with some standard database of vectors \( \dictionary \define \pmat{\dict{1}}{\dict{2}}{\dict{\dsize}} \in \hilbert^{\dsize} \) referred to as the \emph{dictionary}. For a given encoder map \( \repvec \), since every vector \( \sample{} \in \hilbert \) is identified by its code \( \repvec (\samplevec) \), we shall refer to \( \repvec (\samplevec) \) as the \emph{representation} of \( \samplevec \) under the encoder \( \repvec \).  Our objective is to find a dictionary-encoder pair \( (\dictionary\opt , \repvec\opt) \) such that the representation has desirable features like sparsity etc., and the reconstruction is fairly accurate. We do so by formulating an optimization problem, such that the dictionary-encoder pair \( (\dictionary\opt , \repvec\opt) \) obtained from its optimal solution have the desirable features. We refer to the task of finding such a pair as the \emph{Dictionary Learning Problem}, and in short DLP. The DLP is studied in two components namely, the \emph{ecoding problem} and the \emph{dictionary learning} component.

\subsection{The encoding problem} 
The central task in representing a given data optimally, is the \emph{encoding problem}. For a given dictionary \( \dictionary \), the encoding problem is simply the task of encoding a signal \(  \samplevec \in \hilbert \) as another signal \( \repvec_{\dictionary} (\samplevec) \in \R{\dsize} \) such that \( \repvec_{\dictionary} (\samplevec) \) has desirable features like sparsity, and the reconstruction: \( x - \dictionary \repvec_{\dictionary} (\samplevec) \) is within the limits.

We would like to encode such that the reconstruction is similar to the original signal. Therefore, \( \repvec \in \R{\dsize} \) is a feasible representation of \( \samplevec \) if it satisfies \( \inpnorm{\samplevec - \dictionary \repvec} \leq \error (\samplevec) \), where \( \error : \hilbert \longrightarrow [0 , +\infty[ \) is a given \emph{error threshold function}. Some classic examples of error threshold function are
\begin{itemize}
    \item A constant function, where \( \error (x) = \epsilon \) for some \( \epsilon \geq 0 \).
    \item An SNR type function, where \( \error (x) = \epsilon \inpnorm{x} \) for some \( \epsilon \geq 0 \).
\end{itemize}
Ideally, if the data is not noisy, we would want to obtain exact reconstruction of the signals. However, permitting a small amount of error in the reconstruction allows us to encode signals to obtain other desirable features like sparsity etc., in its representation. In other situations like image denoising etc., where the data is corrupted by some bounded noise, \( \error \) is chosen based on the statistical properties of the noise.

It is desirable for the encoded representation to have certain characteristics like sparsity, minimum energy etc. For a given triplet \( (\samplevec , \dictionary , \error) \), there could be many feasible representations. The intended representation that has desirable features is selected by minimizing a certain cost function \( \cost : \R{\dsize} \longrightarrow [0, +\infty[ \). The particular characteristics desired in the representation depends on the type of cost function chosen. For instance, choosing the \( \ell_1 \)-norm : \( \cost (\cdot) = \norm{\cdot}{1} \) induces sparsity in the representation \cite{bach2011convex}, \cite{tibshirani1996regression}, \cite{donoho2006compressed}, and choosing the \( \ell_2 \)-norm : \( \norm{\cdot}{2} \), provides unique representation and group sparsity \cite{meier2008group}. We shall assume the following with regards to the cost function \( \cost \), and these assumptions are in force throughout the article.
\begin{assumption}
\label{assumption:cost-function}
The cost function \( \cost : \R{\dsize} \longrightarrow [0, +\infty[ \) satisfies the following.
\begin{itemize}
        \item \emph{Positive Homogeneity} : For every \( \alpha \geq 0 \) and \( \repvec \in \R{\dsize} \), we have \( \cost (\alpha \repvec ) = \alpha^{\horder} \cost (\repvec) \), where \( \horder > 0 \) is the order of homogeneity.
        \item \emph{Pseudo-Convexity} : The sublevel set \( \costatomset \define \{ \repvec \in \R{\dsize} : \cost (\repvec) \leq 1 \} \) is convex.
        \item \emph{Inf-Compactness} : The set \( \costatomset \) is compact.
\end{itemize}
\end{assumption}
It is easily seen that a cost function that satisfies Assumption \ref{assumption:cost-function} with unit order homogeneity is a \emph{guage} function corresponding to the set \( \costatomset \), and whenever the set \( \costatomset \) is symmetric about the origin, the cost function is a norm.
    
In many scenarios, the final cost function considered is obtained by adding a small penalty function to the actual cost function in order to obtain some other desirable properties. For instance, in the basis pursuit denoising problem, it is customary to add a small \( \ell_2 \)-penalty to the \( \ell_1 \)-cost in order to enforce uniqueness of the optimal solution. By considering a generic definition of cost function \( \cost (\cdot) \) as discussed, such adjustments to the actual cost function are easily incorporated.

Therefore, for a given dictionary \( \dictionary \), the optimal encoder map \( \repvec_{\dictionary} : \hilbert \longrightarrow \R{\dsize} \) is such that, for every \( \samplevec \in \hilbert \), \( \repvec_{\dictionary} (\samplevec) \) is an optimal solution to the following problem.
\begin{equation} 
\label{eq:coding-problem-absolute-error}
	\begin{cases}
	 \begin{aligned}
		& \minimize_{\repvec \; \in \; \R{\dsize}} && \cost ( \repvec ) \\
		& \sbjto && \inpnorm{ \samplevec - \dictionary \repvec } \leq \error (\samplevec) .
	\end{aligned}
	\end{cases}
\end{equation}
An example of the encoding problem which is of practical relevance is the classical \emph{Basis Pursuit Denoising} problem \cite{elad2006image}, \cite{candes2008introduction}, that arises in various applications, and in particular compressed sensing.
\begin{equation}
\label{eq:basis-pursuit-denoising}
\begin{cases}
\begin{aligned}
		& \minimize_{\repvec \; \in \; \R{\dsize}} && \norm{\repvec}{1} \\
		& \sbjto && \inpnorm{ \samplevec - \dictionary \repvec } \leq \error (\samplevec) .
\end{aligned}
\end{cases}
\end{equation}

Even though the problem \eqref{eq:coding-problem-absolute-error} arises naturally in many scenarios, one of the concerns is that the formulation \eqref{eq:coding-problem-absolute-error} lacks regularization. It is often solved by considering its equivalent regularized formulation like \eqref{eq:l1-regularized-sparse-coding}. We remedy this by introducing regularization with parameter \( \regulizer \geq 0 \) that can be set to \( 0 \) if not needed. Therefore, we consider the following encoding problem instead of \eqref{eq:coding-problem-absolute-error}.
\begin{equation} 
	\label{eq:coding-problem}
	\begin{cases}
	  \begin{aligned}
		& \minimize_{( \mathsf{c},\; \repvec ) \; \in \; \R{} \times \R{\dsize}}  && \quad \mathsf{c}^{\horder} \\
		& \sbjto							  &&
		\begin{cases}
			\big( \cost (\repvec) \big)^{1/\horder} \leq \mathsf{c} \\
            \inpnorm{ \samplevec - \dictionary \repvec } \leq \error + \regulizer \mathsf{c} .
		\end{cases}
	\end{aligned}
	\end{cases}
\end{equation}
When \( \regulizer = 0 \), we see that the feasible collection of \( \repvec \) is independent from the variable \( \mathsf{c} \). As a consequence we see that for every feasible \( \repvec \in \R{\dsize} \), the minimization over the variable \( \mathsf{c} \) is achieved for \( \mathsf{c} = \cost (\repvec) \). Thus, the encoding problem \eqref{eq:coding-problem} reduces to the more familiar formulation \eqref{eq:coding-problem-absolute-error}.

It might be surprising at first to see the rather unusual formulation \eqref{eq:coding-problem} of the encoding problem. Our formulation \eqref{eq:coding-problem} makes way for the possibility of \( \regulizer \) taking positive values, due to which we obtain several advantages:
\begin{itemize}[leftmargin = *]
    \item The encoding problem is always strictly feasible, which is easily seen by considering \( \mathsf{c} = \frac{1}{\regulizer} \inpnorm{\samplevec} \) and \( \repvec = 0 \). This is a crucial feature in the initial stages of learning an optimal dictionary, essentially when the data lies in a subspace of lower dimension \( m \), such that \( m, \dsize \ll \dimension \).
    
    \item A positive value of \( \regulizer \) provides regularization in the problem. Thus, one  can harvest the advantages that come from regularization like robustness, well conditioning etc. Even though it a parameter that needs to be learned from the data, in the context of sparse representation however, a small value can be chosen by the user depending on the maximum signal loss that can be tolerated.
    
    \item Considering \( \regulizer > 0 \) in the encoding problem leads to a useful fixed point characterization of the optimal dictionary. Such characterizations also lead to simple online algorithms that learn optimal dictionary.
\end{itemize}

We emphasise that the constraints in the encoding problem \eqref{eq:coding-problem} are convex and the cost function is convex-continuous and \emph{coercive}.\footnote{Recall that a continuous function \( \cost \) defined over an unbounded set \( U \) is said to be coercive in the context of an optimization problem, if :
\[
\lim\limits_{\inpnorm{u} \to \infty } \cost (u) = +\infty \; (- \infty),
\]
in the context of minimization (maximization) of \( \cost \) and the limit is considered from within the set \( U \).
}
Therefore, from the Weierstrass theorem, we conclude that whenever the coding problem is feasible, it admits an optimal solution. To this end, let us define
\begin{equation}
\label{eq:definition-optimal-cost-codes}
\Big( (\samplecost{\dictionary}{\samplevec})^{1/\horder}, \codes \Big) \define
	\begin{cases}
	  \begin{aligned}
		& \argmin_{( \mathsf{c},\; \repvec ) \; \in \; \R{} \times \R{\dsize}}  && \quad \mathsf{c}^{\horder} \\
		& \sbjto							  &&
		\begin{cases}
			\big( \cost (\repvec) \big)^{1/\horder} \leq \mathsf{c} \\
            \inpnorm{ \samplevec - \dictionary \repvec } \leq \error + \regulizer \mathsf{c} .
		\end{cases}
	\end{aligned}
	\end{cases}
\end{equation}
\begin{remark}
We observe that \( \samplecost{\dictionary}{\samplevec} \) is also the optimal value achieved in \eqref{eq:coding-problem}. In view of this, we shall abuse the notation slightly and say that \( \samplecost{\dictionary}{\samplevec} = +\infty \) and \( \codes = \emptyset \) whenever \eqref{eq:coding-problem} is infeasible.
\end{remark}
It should be noted that, both the encoding cost \( \samplecost{\dictionary}{\samplevec} \) and the set of optimal representations \( \codes \) are specific to a given value of regularization parameter \( \regulizer\), cost and error threshold functions \( \cost (\cdot ) \) and \( \error (\cdot) \) respectively, even though it is not specified in their notations.

\begin{definition}
\label{def:representability}
Let \( \dictionary \in \unitdictionaryset \), \( \error : \hilbert \longrightarrow [0, +\infty[ \) and let \( \regulizer \geq 0 \). A vector \( \samplevec \in \hilbert \) is said to be \( \Depsdelta \)-encodable if \( \samplecost{\dictionary}{\samplevec} < +\infty \).
\end{definition}

From the definitions it is immediate that for a given dictionary \( \dictionary \), \( \samplevec \in \hilbert \) is \( \Depsdelta \)-encodable if and only if the corresponding encoding problem \eqref{eq:coding-problem} is feasible. If so, every \( \Depsdelta \)-encodable vector \( \samplevec \) is thus encoded as an element \( \repvec_{\dictionary} (\samplevec) \in \codes \) while incurring a cost of \( \samplecost{\dictionary}{\samplevec} \). Thus, \( \repvec_{\dictionary} : \hilbert \longrightarrow \R{\dsize} \) is an optimal encoder for a given dictionary \( \dictionary \), if \( \repvec_{\dictionary}(\samplevec) \in \codes \) for every \( \samplevec \in \hilbert \). Naturally, the dictionary learning problem is to find a dictionary \( \dictionary\opt \) such that the average encoding cost is minimised.

\subsection{The Dictionary Learning Problem (DLP)} 
Let \( \PP \) be a probability distribution on \( \hilbert \) and \( \rv \) be a \( \PP \)-distributed random variable. Our objective is to find a dictionary that facilitates optimal encoding of the data, which are the samples drawn from \( \PP \). We consider the cost incurred to encode the random variable \( \rv \) using the dictionary \( \dictionary \) to be \( \samplecost{\dictionary}{\rv} \). Therefore, we seek to solve the following problem:
\begin{equation}
\label{eq:dictionary-learning-problem}
        \minimize_{\dictionary \; \in \; \unitdictionaryset} \quad \EE_{\PP} \big[ \samplecost{\dictionary}{\rv} \big] ,
\end{equation}
where \( \unitdictionaryset \) is a compact convex subset of \( \hilbert \). If we allow the dictionary vectors to have arbitrary lengths; every vector in \( \hilbert \) can be written as a linear combination with arbitrarily small coefficients. In the context of dictionary learning, the objective function in  \eqref{eq:dictionary-learning-problem} can be made arbitrarily small by considering a dictionary of arbitrarily long vectors, which makes the problem trivial. Therefore, it is a standard practice to consider an upper bound on the length of each dictionary vector, and for simplicity, this upper bound is chosen to be unity. Thus, the set of feasible dictionaries that is typically considered in practice is
    \begin{equation}
    \label{eq:feasible-dictionaries}
    \unitdictionaryset = \big\{ \pmat{\dict{1}}{\dict{2}}{\dict{\dsize}} : \inpnorm{\dict{i}} \leq 1 \text{ for all } i = 1,2,\ldots,\dsize \big\}.
    \end{equation}
However, in general the feasible dictionary set \( \unitdictionaryset \) could be different depending on applications. For instance, in the non-negative matrix factorization problem the elements of \( \unitdictionaryset \) are further constrained to be element wise non-negative.

For a positive integer \( \horizon \), let \( (\sample{t})_{t = 1}^{\horizon} \) be a given collection of samples drawn from the distribution \( \PP \). For \( t = 1,2,\ldots,\horizon \), let \( \samplecost{\cdot}{t} \define \samplecost{\cdot}{\sample{t}} \), \( \encodermap{\cdot}{t} \define \encodermap{\cdot}{\sample{t}} \) and \( \error_t \define \error (\sample{t}) \), then the dictionary learning problem for the sampled data \( (\sample{t})_t \) is
\begin{equation}
\label{eq:DL-samples}
\minimize_{\dictionary \; \in \; \unitdictionaryset} \ \frac{1}{\horizon} \summ{t = 1}{\horizon} \samplecost{\dictionary}{t} \ .
\end{equation}
In principle, one would want to solve \eqref{eq:dictionary-learning-problem}. However in most practical situations, the knowledge of the entire distribution \( \PP \) is unknown. Often what is available is either a large collection of samples drawn from \( \PP \) or a sequence (likely an iid sequence) of \( \PP \)-distributed samples. In the case when only iid samples drawn from \( \PP \) are available, we solve \eqref{eq:DL-samples} by taking the limit as \( T \longrightarrow +\infty \).

For the special case of \( \regulizer = 0 \), using the definition of the encoding cost \( \samplecost{\dictionary}{t} \), the dictionary learning problem \eqref{eq:DL-samples} reduces to the following more familiar form 
\begin{equation}
\label{eq:DL-conventional}
\begin{cases}
\begin{aligned}
        &\minimize_{\dictionary , \; (\repvec_t)_t}  && \frac{1}{\horizon} \summ{t = 1}{\horizon} \cost (\repvec_t) \\
	    & \sbjto			&&
	    \begin{cases}
	       \dictionary \in \unitdictionaryset , \\
	       \repvec_t \in \R{\dsize} , \text{ for all } t = 1,2,\ldots, \horizon , \\
	       \inpnorm{ \sample{t} - \dictionary \repvec_t } \leq \error_t \text{ for all } t = 1,2,\ldots, \horizon .
	    \end{cases}
	\end{aligned}
\end{cases}
\end{equation}

\section{Main results, algorithms and discussion}
\label{section:main-results-and-algo}
For \( \samplevec \in \hilbert \) such that \( \inpnorm{\samplevec} \leq \error (\samplevec) \), we immediately see that the pair \( \R{}_+ \times \R{\dsize} \ni (\mathsf{c}\opt, \repvec\opt ) \coloneqq (0,0) \) is feasible for \eqref{eq:coding-problem}. Moreover, since \( \cost (\repvec) > 0 \) for every \( \repvec \neq 0 \) we conclude that \( \samplecost{\dictionary}{\samplevec} = 0 \) for every \( \dictionary \in \unitdictionaryset \). As a result, every \( \samplevec \in \hilbert \) satisfying  \( \inpnorm{\samplevec} \leq \error (\samplevec) \) can be optimally represented by the zero vector \( 0 \in \R{\dsize} \) irrespective of the dictionary. Consequently, such samples do not play any role in the optimization over the dictionary variable \( \dictionary \) and can be ignored. Therefore, we assume w.l.o.g. that \( \inpnorm{\sample{t}} > \error_t \) for all \( t = 1,2,\ldots,T \) in the dictionary learning problem \eqref{eq:DL-samples}.

For every \( \samplevec \in \hilbert \) satisfying \( \inpnorm{\samplevec} > \error (\samplevec) \), let the function \( J_{\samplevec} : \unitdictionaryset \times \costatomset \longrightarrow [0 , +\infty] \) be defined by
\begin{equation}
\label{eq:J-map-def}
J_{\samplevec} (\dictionary, h) \define 
\begin{cases}
\begin{aligned}
& \sup_{\lambda} \; && r(\horder) \Big( \inprod{\lambda}{\samplevec} - \error (\samplevec) \inpnorm{\lambda} \Big)^{q(\horder)} - \Big( \regulizer \inpnorm{\lambda} + \inprod{\lambda}{\dictionary h} \Big) \\
& \text{s.t.} && \inprod{\lambda}{\samplevec} - \error \inpnorm{\lambda} >  0 ,
\end{aligned}
\end{cases}    
\end{equation}
where \( r(\horder) = \horder (1 + \horder) \) and \( q(\horder) = \frac{\horder}{1 + \horder} \). Notice that for every fixed \( \hvar{} \in \costatomset \), the objective function in \eqref{eq:J-map-def} is linear w.r.t. the dictionary variable \( \dictionary \), and thus, also convex in \( \dictionary \). Since \( J_{\samplevec} (\dictionary , \hvar{} ) \) is a pointwise supremum of this objective function, we conclude that the map \( \unitdictionaryset \ni \dictionary \longmapsto J_{\samplevec} (\dictionary , \hvar{}) \) is convex for every \( \hvar{} \in \costatomset \). From similar arguments, it also follows that the map \( \costatomset \ni \hvar{} \longmapsto J_{\samplevec} (\dictionary , \hvar{}) \) is convex for every \( \dictionary \).

The current formulation of the DLP \eqref{eq:DL-samples} is ill-posed in the sense that it does not admit a straightforward alternating minimization scheme which minimizes w.r.t. the variables \( (\repvec_t)_t \) and \( \dictionary \) while keeping the other fixed alternatingly. We resolve this issue by proposing an equivalent optimization problem in terms of the function \( J_{\samplevec} (\dictionary, \hvar{}) \), and show that it is equivalent to \eqref{eq:DL-samples} but also well posed. To this end, let \( J_t (\dictionary , \hvar{}) \define J_{\sample{t}} (\dictionary , \hvar{}) \) for every \( t = 1,2,\ldots,T \), and consider the optimization problem
\begin{equation}
\label{eq:DLP-in-J}
\begin{cases}
\begin{aligned}
& \minimize_{\dictionary , \; (\hvar{t})_t} && \frac{1}{T} \sum_{t = 1}^T J_t (\dictionary , \hvar{t}) \\
& \sbjto && 
\begin{cases}
\dictionary \in \unitdictionaryset \\
\hvar{t} \in \costatomset \quad \text{for all } t = 1,2,\ldots,T.
\end{cases}
\end{aligned}
\end{cases}
\end{equation}

\begin{lemma}
\label{lemma:equivalent-DL}
Consider the dictionary learning problem \eqref{eq:DL-samples} and \eqref{eq:DLP-in-J} for the given data  \( (\sample{t})_t \) such that \( \inpnorm{\sample{t}} > \error_t \) for every \( t = 1,2,\ldots,T \). The optimization problem \eqref{eq:DLP-in-J} is equivalent to the dictionary learning problem \eqref{eq:DL-samples} in the sense that 
\begin{enumerate}[leftmargin = *, label = \rm{(\roman*)}]
\item the optimal values of \eqref{eq:DL-samples} and \eqref{eq:DLP-in-J} are identical

\item every optimal solution \( (\dictionary\opt, (\hvar{t}\opt)_t) \) to \eqref{eq:DLP-in-J} can be computed from one of the optimal solution \( (\dictionary\opt, \repvec_t\opt)_t \) to \eqref{eq:DL-samples}, and vice versa.
\end{enumerate}
\end{lemma}
\begin{proof}
For a fixed \( \dictionary \), observe that minimization over the variables \( (\hvar{t}')_t \) in \eqref{eq:DLP-in-J} is separable, and the joint minimization problem separates into the individual problems: \( \min\limits_{\hvar{t} \; \in \; \costatomset} \ J_t (\dictionary , \hvar{t}) \) for each \( t = 1,2 \ldots, \horizon \). Substituting from the definition \eqref{eq:J-map-def} of \( J_t (\dictionary , \hvar{t}) \), we see that solving the individual problems for each \( t \) is equivalent to solving the min-max problem
\begin{equation}
\label{eq:coding-problem-inf-sup-equivalent}
\begin{cases}
\begin{aligned}
& \min_{\hvar{t} \; \in \; \costatomset} \; \sup_{\separator{t}} \ \ && r (\horder) \Big( \inprod{\separator{t}}{\sample{t}} - \error_t \inpnorm{\separator{t}} \Big)^{q (\horder)} - \Big( \regulizer \inpnorm{\separator{t}} + \inprod{\separator{t}}{\dictionary \hvar{t}} \Big) \\
& \sbjto &&  \inprod{\separator{t}}{\sample{t}} - \error_t \inpnorm{\separator{t}} > 0 \; .
\end{aligned}
\end{cases}
\end{equation}
We know from \cite{sheriff2019LIP} that the optimal value of the min-sup problem \eqref{eq:coding-problem-inf-sup-equivalent} is equal to the encoding cost \( \encodedcost{\dictionary}{\sample{t}} \). Therefore, solving for the minimization over \( (\hvar{t})_t \) in \eqref{eq:DLP-in-J} we immediately see that \eqref{eq:DLP-in-J} reduces to
\[
\min_{\dictionary \in \; \unitdictionaryset} \ \frac{1}{\horizon} \summ{t = 1}{T} \samplecost{\dictionary}{t} \; .
\]
Thus, the optimal values of \eqref{eq:DL-samples} and \eqref{eq:DLP-in-J} and their respective set of optimal dictionaries are identical.

From \cite{sheriff2019LIP}, we also know that for each \( t = 1,2,\ldots,T \), the minimization over the variable \( \hvar{t} \) is achieved and \( H_t (\dictionary) \define (\samplecost{\dictionary}{t})^{-1/\horder} \cdot \encodermap{\dictionary}{t} \) is the set of minimizers. Suppose that \( (\dictionary\opt , (\repvec_{t}\opt)_t ) \) is an optimal solution to the dictionary learning problem \eqref{eq:DL-samples}, then it immediately follows that \( (\dictionary\opt , (\hvar{t}\opt)_t ) \) is an optimal solution to \eqref{eq:DLP-in-J} if \( \hvar{t}\opt = (\samplecost{\dictionary\opt}{t})^{-1/\horder} \cdot \encodermap{\dictionary\opt}{t} \) for each \( t \). Similarly, if \( (\dictionary' , (\hvar{t}')_t ) \) is an optimal solution to \eqref{eq:DLP-in-J}, we first see that 
\[
J_t (\dictionary' , \hvar{t}') \; = \; \min_{\hvar{t} \in \costatomset} \; J_t (\dictionary' , \hvar{t}) = \samplecost{\dictionary'}{t} \; ,
\]
which immediately implies that \( (\dictionary' , (\repvec'_t)_t) \) is an optimal solution to the dictionary learning problem \eqref{eq:DL-samples} if \( \repvec'_t = \big( J_t (\dictionary' , \hvar{t}') \big)^{1/\horder} \cdot \hvar{t}' \) for every \( t = 1,2,\ldots,T \).
\end{proof}

We propose to solve the dictionary learning problem \eqref{eq:DL-samples} by finding a solution to its equivalent problem \eqref{eq:DLP-in-J} instead. A simple and straightforward method to apply is to minimize over one of the variables \( (\hvar{t})_t \) and \( \dictionary \) by keeping the other one fixed and then alternate.
\begin{algorithm}[h]
\caption{Batch-wise alternating minimization algorithm to solve the DLP}
\label{algo:general-DL-algo}
\KwIn{The data \( \mathcal{X} \subset \hilbert \) which is a finite collection of points, a positive integer \( \dsize \), cost and error threshold functions
\( \cost \) and \( \error \) respectively, the regularizer \( \regulizer \geq 0 \).}
\KwOut{A dictionary \( \dictionary\opt \) which is at least a stationary point to \eqref{eq:DL-samples} and the corresponding representation vectors \( (\repvec_t\opt)_t \) for the data.}

\nl \textit{Remove irrelevant samples} :  Discard every \( \samplevec \in \mathcal{X} \) that satisfies \( \inpnorm{\samplevec} \leq \error (\samplevec) \), and let \( (\sample{t})_{t = 1}^{\horizon} \) be the remaining true data samples.

\nl \textit{Initialization} : Compute or generate an initial dictionary \( \dictionary_0 \).

\nl Set \( \dictionary' \longleftarrow \dictionary_0 \).

\nl \textbf{Iterate till stopping criteria is met.}

\textit{Updating the variables \( (\hvar{t})_t \)} : Compute
 \begin{equation}
    \label{eq:alternating-minimization-LIP}
     (\hvar{t}')_t \; \in \; \argmin_{(\hvar{t})_t \; \subset \; \costatomset} \ \frac{1}{T} \summ{t = 1}{T} J_t (\dictionary' , \hvar{t} ) ,
 \end{equation}
by either solving the encoding problem \eqref{eq:coding-problem} directly or its equivalent min-max form \eqref{eq:coding-problem-inf-sup-equivalent} via Algorithm \ref{algo:LIP} for each \( t = 1,2,\ldots,T \).
 
\textit{Updating the dictionary} : Using the collection \( (\hvar{t}')_t \) computed from \eqref{eq:alternating-minimization-LIP} solve the optimization problem via Algorithm \ref{algo:D-update}
\begin{equation}
\label{eq:alternating-minimization-DL}    
\dictionary' \; \in \; \argmin_{\dictionary \; \in \; \unitdictionaryset} \; \frac{1}{T} \summ{t = 1}{T} J_t (\dictionary , \hvar{t}')
\end{equation}
 
\nl \textbf{Repeat}

\nl \textbf{Output the dictionary and codes.}
\end{algorithm}

Observe that due to convexity of \( J_t(\dictionary , \hvar{t}) \) in individual arguments, both of the optimization problems \eqref{eq:alternating-minimization-LIP} and \eqref{eq:alternating-minimization-DL} are convex. Moreover, we shall establish that both of these minimization problems and particularly, the optimization over the dictionaries in \eqref{eq:alternating-minimization-DL} admit a well defined and non-trivial optimal solution. This is in complete contrast to the ill-posedness of the original formulations \eqref{eq:DL-conventional}, \eqref{eq:DL-SCP} where no such meaningful alternating minimization techniques exist.

\subsubsection*{Minimization w.r.t. the variables \( (\hvar{t})_t \)} 

\noindent For a given dictionary \( \dictionary' \), consider the problem \eqref{eq:alternating-minimization-LIP}. We recall from the proof of Lemma \ref{lemma:equivalent-DL} that the minimization over variables \( (\hvar{t})_t \) is separable, and for each \( t = 1,2,\ldots,T \), the problem \( \min\limits_{\hvar{t} \in \costatomset} \ J_t (\dictionary' , \hvar{t}) \) results in solving the min-max problem \eqref{eq:coding-problem-inf-sup-equivalent} by computing a saddle point via Algorithm \ref{algo:LIP}. We have reproduced the same algorithm to compute a saddle point of min-max problems equivalent to LIPs from Chapter 2 for convenience.

\begin{algorithm}[h]
\label{algo:LIP}
\caption{Projected gradient descent algorithm to solve \eqref{eq:alternating-minimization-LIP}}
\KwIn{Problem data: \( \samplevec, \ \dictionary , \ \error , \ \regulizer , \ \cost \).}
\KwOut{An optimal solution \( \repvec \in \codes \) and the optimal value \( \samplecost{\dictionary}{\samplevec} \).}

\nl Proceed only if \( \inpnorm{\samplevec} > \error \), else output \( 0 \).

\nl Initialize \( \hvar{} \) and \( \separator{} \).

\nl \textbf{Iterate till convergence}

\quad Iterate \( M \) times
\[
\separator{} \longleftarrow \separator{} \; + \;  \alpha \left( \frac{\horder^2 \big( \samplevec - \frac{\error}{\inpnorm{\separator{}}} \separator{} \big)}{\big( \inprod{\separator{}}{\samplevec} - \error \inpnorm{\separator{}} \big)^{\frac{1}{1 + \horder}}} \; - \; \frac{\regulizer}{\inpnorm{\separator{}}} \separator{}  \; - \; \dictionary (\hvar{}) \right)
\]

\quad \emph{Update} : \( \hvar{} \longleftarrow \pi_{\cost} \left(  \hvar{} \; + \; \beta \big( \dictionary\transp (\separator{}) \big) \right) \)

\nl \textbf{Repeat}

\nl \textbf{Output}: \( \samplecost{\dictionary}{\samplevec} = \inprod{\separator{}}{\dictionary (\hvar{})} \) and \( \repvec = \samplecost{\dictionary}{\samplevec} \cdot \hvar{} \).

\end{algorithm}

An interesting observation to be made here is that the set of minimizers \( H_t (\dictionary') \) to \eqref{eq:alternating-minimization-LIP} satisfies \( H_t (\dictionary') \; = \big(\samplecost{\dictionary'}{t} \big)^{1/\horder} \cdot \encodermap{\dictionary'}{t} \). Therefore, if we have access to a black box which solves the encoding problem \eqref{eq:coding-problem} and provides us an element \( \repvec'_t \in \encodermap{\dictionary}{t} \) and the optimal cost \( \samplecost{\dictionary'}{t} \), a solution \( \hvar{t}' \in H_t(\dictionary') \) can be obtained simply by scaling the black box solution \( \repvec'_t \) appropriately. Since the subsequent step to compute a ``good'' dictionary requires a solution \( (\hvar{t}')_t \in H_t (\dictionary') \) and not the black box solution \( \repvec'_t \), the fact that they are scalar multiples of each other nullifies the need to solve the min-max problem again only to compute \( \hvar{t}' \).

\subsubsection*{Minimization w.r.t. the dictionary variable.}

\noindent For a given collection \( (\hvar{t})_t \subset \costatomset \), let us consider the dictionary update step \eqref{eq:alternating-minimization-DL} in Algorithm \ref{algo:general-DL-algo}.  On substituting for \( J_t \) from \eqref{eq:J-map-def}, the minimization problem \eqref{eq:alternating-minimization-DL} over the dictionaries becomes
\[
\min_{\dictionary \; \in \; \unitdictionaryset} \ \frac{1}{\horizon} \summ{t = 1}{\horizon} 
\begin{cases}
\begin{aligned}
& \sup_{\separator{t} } \; && r(\horder) \Big( \inprod{\separator{t}}{\sample{t}} -  \error_t \inpnorm{\separator{t}} \Big)^{q(\horder)} - \Big( \regulizer \inpnorm{\separator{t}} + \inprod{\separator{t}}{\dictionary \hvar{t}} \Big) \\
& \text{s.t.} && 
\inprod{\separator{t}}{\sample{t}} - \error_t \inpnorm{\separator{t}} >  0 \; .
\end{aligned}
\end{cases}
\]
Observe that for each \( t = 1,2,\ldots,\horizon \), the maximization over \( \separator{t} \) is independent of the others, and therefore, these individual maximization problems can be clubbed together and written as the min-sup problem
\begin{equation}
\label{eq:DL-D-update-min-max}
\begin{cases}
\begin{aligned}
& \min_{\dictionary \; \in \; \unitdictionaryset} \ \sup_{(\separator{t})_t} \; 
&& \frac{1}{\horizon} \summ{t = 1}{\horizon} 
\begin{cases}
\begin{aligned}
r(\horder) \Big( \inprod{\separator{t}}{\sample{t}} & - \error_t \inpnorm{\separator{t}} \Big)^{q(\horder)} \\
 & - \Big( \inprod{\separator{t}}{\dictionary \hvar{t}} + \regulizer \inpnorm{\separator{t}} \Big)
\end{aligned}
\end{cases}
\\
& \sbjto && 
\inprod{\separator{t}}{\sample{t}} - \error_{t} \inpnorm{\separator{t}} >  0 \quad \text{for all } t = 1,2,\ldots, \horizon .
\end{aligned}
\end{cases}
\end{equation}
It is easily verified that the objective function in the min-sup problem above is convex in the minimizing variable \( \dictionary \), and since \( q (\horder) = \frac{\horder}{1 + \horder} \in \; ]0 , 1[ \) it is also jointly concave in the maximizing variables \( (\separator{t})_t \). Moreover, since the constraints \( \inprod{\separator{t}}{\sample{t}} - \error_t \inpnorm{\separator{t}} > 0 \) for all \( t = 1,2,\ldots, \horizon \) and \( \dictionary \in \unitdictionaryset \), are convex, the min-sup problem is a convex program. We have the following main result with regards to the existence of a solution to the min-sup problem.

\begin{proposition}
\label{proposition:dictionary-update-min-max}
Let the given data \( (\sample{t})_t \) be such that \( \inpnorm{\sample{t}} > \error_t \) for all \( t = 1,2,\ldots,T \). For a given collection \( (\hvar{t})_t \subset \costatomset \) and real numbers \( r > 0 \), \( q \in ]0 , 1[ \), consider the following min-sup problem
\begin{equation}
\label{eq:dictionary-update-min-max}
\begin{cases}
\begin{aligned}
& \min_{\dictionary \; \in \; \unitdictionaryset} \; \sup_{(\separator{t})_t } \; && \frac{1}{\horizon} \summ{t = 1}{\horizon}
\begin{cases}
\begin{aligned}
r \Big( \inprod{\separator{t}}{\sample{t}} & - \error_t \inpnorm{\separator{t}} \Big)^{q} \\
& - \Big( \regulizer \inpnorm{\separator{t}} + \inprod{\separator{t}}{\dictionary \hvar{t}} \Big)
\end{aligned}
\end{cases}
\\
& \sbjto && 
\inprod{\separator{t}}{\sample{t}} - \error_t \inpnorm{\separator{t}} >  0 \; \text{\rm{ for all} } t = 1,2,\ldots, T .
\end{aligned}
\end{cases}    
\end{equation}
If \( \regulizer > 0 \), the min-sup problem admits a unique saddle point solution \( (\dictionary' , (\separator{t}')_t ) \).
\end{proposition}

\begin{remark}
If \( \regulizer = 0 \), existence of a saddle point solution to the min-max problem \eqref{eq:dictionary-update-min-max} depends on the data \( (\sample{t})_t \) and \( (\hvar{t})_t \). We would like to emphasize that the min-sup problem \eqref{eq:dictionary-update-min-max} could potentially have a saddle point solution even if \( \regulizer = 0 \). However, it is very difficult to characterize under what conditions on \( (\sample{t})_t \) and \( (\hvar{t})_t \), such a saddle point solution exists. In practice however, it is observed that even if \( \regulizer = 0 \), the min-sup problem usually admits a saddle point.
\end{remark}

\begin{remark}
If \( \regulizer = 0 \), depending on the data \( (\sample{t})_t \), there could potentially exist a collection \( (\hvar{t})_t \) such that the value of the corresponding min-sup problem \eqref{eq:dictionary-update-min-max} is unbounded. However, it is always finite if we consider the collection \( (\hvar{t})_t \) such that \( \hvar{t} \in H_t (\dictionary') \) for some dictionary \( \dictionary' \). Therefore, it is a good practice to start the alternating minimization by first obtaining an optimal collection of \( (\hvar{t}')_t \) using some dictionary \( \dictionary' \) followed by the dictionary update step.
\end{remark}

\begin{algorithm}[h]
\label{algo:D-update}
\caption{The dictionary update algorithm}
\KwIn{The data \( (\sample{t})_t \subset \hilbert \), non-negative real numbers \( (\error_t)_t \), the regularizer \( \regulizer \geq 0 \), and a collection\( (\hvar{t})_t \).}
\KwOut{A saddle point solution \( (\dictionary' , (\separator{t}')_t ) \) to \eqref{eq:DL-D-update-min-max}}

\nl \textbf{Iterate till convergence}

\quad Compute \( \dictionary = \pmat{\dict{1}}{\dict{2}}{\dict{\dsize}} \) such that 
\begin{equation}
\label{eq:optimal_d-for-given-lambda}
d(i) \; = \; \frac{  \frac{1}{T} \summ{t = 1}{T} \hvar{t}(i) \separator{t} }{ \inpnorm{  \frac{1}{T} \summ{t = 1}{T} \hvar{t}(i) \separator{t} } } \text{ for every \( i = 1,2,\ldots,\dsize \).}      
\end{equation}

\quad For each \( t = 1 ,2, \ldots,T \), update
\begin{equation}
\label{eq:lambda-update}
\begin{aligned}
\dualvar{t} \; & \longleftarrow \; \big( \inprod{\separator{t}}{\sample{t}} - \error_t \inpnorm{\separator{t}} \big)^{\frac{1}{1 + \horder}} \text{ and} \\
\separator{t} \; & \longleftarrow \; \separator{t} + \frac{\alpha}{\dualvar{t}} \Big( \horder^{\frac{1}{1+\horder}} \sample{t} \; - \; \dualvar{t} \dictionary \hvar{t} \; - \; \frac{\error_t + \regulizer \dualvar{t}}{\inpnorm{\separator{t}}} \separator{t} \Big) \; .
\end{aligned}
\end{equation}

\nl\textbf{Repeat}

\nl\textbf{Output} : \( \dictionary' = \dictionary \) and \( (\separator{t})_t = (\separator{t})_t \).

\end{algorithm}

Since the order of optimization in \eqref{eq:alternating-minimization-DL} can be changed from  min-max to max-min, we see that the minimization over the dictionaries for a given sequence \( (\separator{t})_t \) can be explicitly solved, and it is achieved at the unique dictionary given by \eqref{eq:optimal_d-for-given-lambda}. The resulting maximization problem in variables \( (\separator{t})_t \) can then be solved using any of the optimization algorithms. As a representative algorithm, the update \eqref{eq:lambda-update} in Algorithm \ref{algo:D-update} performs a gradient ascent on the variable \( \separator{t} \). Instead of this, one can implement other update schemes like accelerated gradient ascent, depending on the specifics of the problem at hand.

If the feasible set of dictionaries is other that the standard one \eqref{eq:feasible-dictionaries}, neither an explicit solution like that of \eqref{eq:optimal_d-for-given-lambda} nor uniqueness can be guaranteed for the minimization problem over dictionaries for a given sequence \( (\separator{t})_t \). In such cases, saddle point seeking projected descent-ascent schemes to solve min-max problems can be implemented. In a simple gradient based descent-ascent scheme, the updates \eqref{eq:optimal_d-for-given-lambda} and \eqref{eq:lambda-update} are replaced with
\[
\begin{aligned}
\dictionary \; & \longleftarrow \pi_{\unitdictionaryset} \Big( \dictionary + \beta \summ{t = 1}{T} \separator{t}\hvar{t}\transp \Big) \\
\dualvar{t} \; & \longleftarrow \; \big( \inprod{\separator{t}}{\sample{t}} - \error_t \inpnorm{\separator{t}} \big)^{\frac{1}{1 + \horder}} \text{ and} \\
\separator{t} \; & \longleftarrow \; \separator{t} + \frac{\alpha}{\dualvar{t}} \Big( \horder^{\frac{1}{1+\horder}} \sample{t} \; - \; \dualvar{t} \dictionary \hvar{t} \; - \; \frac{\error_t + \regulizer \dualvar{t}}{\inpnorm{\separator{t}}} \separator{t} \Big) \text{ for every } t .
\end{aligned}
\]
For large number of samples, i.e., when \( T \) is large, the dictionary updating step-size \( \beta \) needs to be slower than the \( \separator{t} \) updating step-size \(  \alpha \).

\subsection{Optimality conditions}
We provide necessary conditions for a dictionary to be optimal and sufficient conditions for the stationarity of a given dictionary. Due to non-convexity of the dictionary learning problem \eqref{eq:DL-samples}, sufficient conditions can only guarantee stationarity.

\begin{definition}
\label{def:optimal-separator-set}
Let \( \dictionary \in \unitdictionaryset \), the cost and error threshold functions \( \cost , \error \), and \( \regulizer \geq 0 \) be given. Then for every \( \samplevec \in \hilbert \) that is \( \Depsdelta \)-encodable, let \( \Lambda (\dictionary , \samplevec) \subset \hilbert \) denote the collection of points \( \separator{} \in \hilbert \setminus \closednbhood{0}{\error} \) that satisfy the following two conditions simultaneously:
\begin{itemize}
\item \( \inprod{\separator{}}{\samplevec} - \error \inpnorm{\separator{}} \ = \ (\samplecost{\dictionary}{\samplevec})^{1/\horder} \), and
\item \( 1 \; = \; \regulizer \inpnorm{\separator{}} + \max\limits_{\hvar{} \in \costatomset} \inprod{\separator{}}{\dictionary \hvar{}} \) .
\end{itemize}
\end{definition}

Recall from Chapter 2 that the set \( \separatorset{\dictionary}{\sample{}} \) is a scalar multiple of to the optimal solutions to the Fenchel Dual problem of the encoding problem \eqref{eq:coding-problem}. A complete description of the set \( \separatorset{\dictionary}{\sample{t}} \) is available in \cite[Proposition 3.10]{sheriff2019LIP}.

\begin{proposition}
\label{proposition:optimality-condition}
Consider the dictionary learning problem \eqref{eq:DL-samples}, and let \( \unitdictionaryset \) be any compact convex subset of \( \hilbert \) rather than the standard candidate \eqref{eq:feasible-dictionaries}. The following optimality conditions hold
\begin{enumerate}[label = \rm{(\roman*)}, leftmargin = *]
\item Necessary Condition : If \( \dictionary\opt \in \unitdictionaryset \) is an optimal dictionary for the DLP \eqref{eq:DL-samples} and \( \Lambda (\dictionary\opt , \sample{t}) \neq \emptyset \) for all \( t = 1,2,\ldots,\horizon \). Then for any collection \( ( \hvar{t}')_t \) satisfying \( \hvar{t}' \in \unitencoder{\dictionary\opt}{t} \) for every \( t = 1,2,\ldots,\horizon \), there exists \( (\separator{t}')_t \) satisfying \( \separator{t}' \in \samplecost{\dictionary\opt}{t} \cdot \separatorset{\dictionary\opt}{\sample{t}} \) for every \( t = 1,2,\ldots,T \) such that the following holds
\begin{equation}
\label{eq:optimality-condition}
\dictionary\opt \; \in \; \argmax_{\dictionary \; \in \; \unitdictionaryset} \quad \trace \left( \dictionary \cdot \frac{1}{\horizon} \summ{t = 1}{\horizon} \samplecost{\dictionary\opt}{t} \big( \separator{t}' {\hvar{t}'}\transp \big)  \right) .
\end{equation}

\item Sufficient Condition : Let \( \dictionary' \in \unitdictionaryset \) be such that \( \Lambda (\dictionary' , \sample{t}) \neq \emptyset \) for all \( t = 1,2,\ldots,\horizon \). If there exists \( \big( \hvar{t}' , \separator{t}' \big) \in \unitencoder{\dictionary'}{t} \; \times \; \samplecost{\dictionary'}{t} \Lambda (\dictionary' , \sample{t}) \) for every \( t = 1,2,\ldots,\horizon \) such that the triplet \( (\dictionary' , (\hvar{t}')_t , (\separator{t}')_t) \) satisfies \eqref{eq:optimality-condition}, then \( \dictionary' \) is a stationary point for the DLP \eqref{eq:DL-samples}.
\end{enumerate}
\end{proposition}

\begin{corollary}
Consider the dictionary learning problem \eqref{eq:DL-samples} with \( \regulizer > 0 \) and \( \unitdictionaryset = {\closednbhood{0}{1}}^{\dsize} \). Let \( \dictionary \in {\closednbhood{0}{1}}^{\dsize} \) be given, then the following are equivalent
\begin{enumerate}[label = \rm{(\roman*)} , leftmargin = *]
\item The dictionary \( \dictionary \in \unitdictionaryset \) is a stationary point for the DLP \eqref{eq:DL-samples}.

\item For each \( t = 1,2,\ldots,\horizon \) there exists \( \repvec_t \in \encodermap{\dictionary}{t} \) such that for every \( i = 1,2,\ldots,\dsize \) the following holds
\begin{equation}
\label{eq:fixed-point-delte-positive}
d (i) = \proj_{\closednbhood{0}{1}} \left( \frac{1}{\horizon} \summ{t = 1}{\horizon} \; \frac{ \repvec_t (i) \big( \samplecost{\dictionary'}{t}\big)^{1 - \frac{1}{\horder}} }{ \inprod{\sample{t} - \dictionary \repvec_t}{\dictionary \repvec_t} } \;\Big( \sample{t} - \dictionary \repvec_t \Big) \right) .\footnote{Recall that \( \proj_{S} \) is the projection operator on the set \( S \). In particular, we have \( \proj_{\closednbhood{0}{1}} (v) = \frac{v}{\inpnorm{v}} \) for every \( v \notin \closednbhood{0}{1} \).}    
\end{equation}
\end{enumerate}
\end{corollary}

\subsubsection*{An online algorithm to solve the dictionary learning problem}

\noindent The fixed point characterizations \eqref{eq:optimality-condition} and \eqref{eq:fixed-point-delte-positive} give us an approach to compute a dictionary by means of finding their fixed point solutions. Existing stochastic approximation techniques like Robbins-Monro scheme can be implemented to find a fixed point of the equations \eqref{eq:optimality-condition} or \eqref{eq:fixed-point-delte-positive}. Interestingly, such ideas allow the possibility of an online implementation, where the algorithm operates on a finite but small batch of data to compute an update to the dictionary. This is extremely crucial in situations where the user does not have access to the entire data and instead has access to only get a few iid samples drawn from the data (or the distribution \( \PP \)). 

\begin{algorithm}[h]
\label{algo:online-DL}
\caption{Online dictionary update algorithm}
\KwIn{An iid sequence of data \( (\sample{t})_t \subset \hilbert \), non-negative real numbers \( (\error_t)_t \), the regularizer \( \regulizer \geq 0 \).}
\KwOut{A dictionary \( \dictionary' \) that is at least a stationary point of \eqref{eq:dictionary-learning-problem}.}

\nl Initialise \( t = 0 \), \( \dictionary_0 \).

\nl \textbf{Iterate till convergence}

Compute \( (\repvec'_{t} \in \unitencoder{\dictionary_t}{t} \) and \( \separator{t}' \in \separatorset{\dictionary_t}{t} \) by finding a saddle point to the min-max problem \eqref{eq:coding-problem-inf-sup-equivalent} for the sample \( \sample{t} \) using the dictionary \( \dictionary_t \) via Algorithm \ref{algo:LIP}.

Update the dictionary 
\[
\dictionary_{t + 1} \; = \pi_{\unitdictionaryset} \Big( \dictionary_t + \alpha_t \samplecost{\dictionary_t}{t} \big( \separator{t}' {\repvec'_t}\transp \big) \Big)
\]

\( t \longleftarrow t + 1 \)

\nl\textbf{Repeat}

\nl\textbf{Output} : \( \dictionary' = \dictionary_t \).
\end{algorithm}

When \( \regulizer > 0 \), since the set \( \separatorset{\dictionary_t}{t} \) is a scalar multiple of the singleton \( \{ \sample{t} - \dictionary_t \repvec'_t \} \), then online dictionary update becomes
\[
\dictionary_{t + 1} = \pi_{\unitdictionaryset} \Big( \dictionary_t \; + \; \alpha_t \; \samplecost{\dictionary_t}{t} \big( (\sample{t} - \dictionary_t \repvec_t') {\repvec_t'}\transp \big) \Big) \; ,
\]
where the step-size sequence \( (\alpha_t)_t \) satisfy \( \summ{t = 1}{+\infty} \alpha_t \; = \; +\infty \) and \( \summ{t = 1}{+\infty} \alpha^2 \; < \; +\infty \).

On simplifying for each \( i = 1,2,\ldots,\dsize \), the update for the dictionary vector \( d_t(i) \) at time \( t \) becomes
\begin{equation}
\label{eq:D-element-wise-online-update}
d_{t + 1}(i) = \pi_{\closednbhood{0}{1}} \left\{ d_t(i) \; + \; \alpha'_t \repvec_t (i) (\sample{t} - \dictionary_t \repvec_t) \right\} ,
\end{equation}
where \( (\alpha'_t)_t \) is the step-size sequence obtained by absorbing some scalars. The update scheme is reminiscent of the stochastic sub-gradient descent scheme employed in learning dictionary for conventional formulation. 

Observe in \eqref{eq:D-element-wise-online-update} that only those dictionary vectors which contribute non-trivially in representing the sample \( \sample{t} \) at time \( t \) are updated, i.e., only those vectors for which the corresponding co-efficient \( \repvec_t (i) \neq 0 \) is updated. As a result, the dictionary update can be done in an asynchronous fashion. An intuitive way to understand \eqref{eq:D-element-wise-online-update} is that the dictionary \( d_t (i) \) which is updated at time \( t \) is being pushed towards the sample \( \sample{t} \), and the magnitude of the push is proportional to the contribution of the dictionary in representation, which is \( \abs{ \repvec_t (i) } \).

\section{proofs}
\label{section:proofs}
\begin{proof}[Proof of Proposition \ref{proposition:dictionary-update-min-max}]
We see that for every collection \( (\separator{t})_t \), the maximization problem
\[
\max_{\dictionary \; \in \; \unitdictionaryset} \ \frac{1}{T} \summ{t = 1}{T} \inprod{\separator{t}}{\dictionary \hvar{t}'} ,
\]
admits the unique solution \( \dictionary ((\separator{}t)_t)  = \pmat{d'(1)}{d'(2)}{d'(\dsize)} \) given by
\begin{equation}
\label{eq:optimal-D-for-lambda}
d'(i) \; = \; \frac{  \frac{1}{T} \summ{t = 1}{T} \hvar{t}'(i) \separator{t} }{ \inpnorm{  \frac{1}{T} \summ{t = 1}{T} \hvar{t}'(i) \separator{t} } } \text{ for every \( i = 1,2,\ldots,\dsize \),}
\end{equation}
with an optimal value of \( \summ{i = 1}{\dsize} \inpnorm{  \frac{1}{T} \summ{t = 1}{T} \hvar{t}'(i) \separator{t} } \),
where \( (\hvar{t}'(i))_i \) are the individual components of the vector \( \hvar{t}' \).

Since the objective function in the min-sup problem \eqref{eq:DL-D-update-min-max} is convex in the argument \( \dictionary \) and concave jointly in the arguments \( (\separator{t})_t \), and the set \( \unitdictionaryset \) is compact, the min-max equality follows from Sion's min-max theorem. This allows us to interchange the order of optimization in \eqref{eq:DL-D-update-min-max} and to consider minimizing w.r.t. the dictionaries first. Doing so, and using the dictionary from \eqref{eq:optimal-D-for-lambda}, we conclude that the resulting maximization problem over variables \( (\separator{t})_t \) is
\begin{equation}
\label{eq:inactive-lambda-problem}
\begin{cases}
\begin{aligned}
& \sup_{(\separator{t})_t} \; 
&& \frac{1}{\horizon} \summ{t = 1}{\horizon} 
\begin{cases}
\begin{aligned}
r(\horder) \big( \inprod{\separator{t}}{\sample{t}} & - \error_t \inpnorm{\separator{t}} \big)^{q(\horder)}  \\
& - \; \regulizer \inpnorm{\separator{t}} - \; \summ{i = 1}{\dsize} \inpnorm{  \frac{1}{T} \summ{t = 1}{T} \hvar{t}'(i) \separator{t} }
\end{aligned}
\end{cases}
\\
& \sbjto && 
\inprod{\separator{t}}{\sample{t}} - \error_{t} \inpnorm{\separator{t}} >  0 \quad \text{for all } t = 1,2,\ldots, \horizon .
\end{aligned}
\end{cases}
\end{equation}
To conclude the proposition, we first show that the maximization problem \eqref{eq:inactive-lambda-problem} admits a unique optimal solution \( (\separator{t}')_t \) and then show that  \( (\separator{t}')_t \) together with \( \dictionary ( (\separator{t}')_t ) \) is a saddle point to the min-max problem \eqref{eq:DL-D-update-min-max}. 

Since \( q(\horder) \in ]0,1[ \) and \( \regulizer > 0 \), we know that the sub-linear component \( r(\horder) \big( \inprod{\separator{t}}{\sample{t}} - \error_t \inpnorm{\separator{t}} \big)^{q(\horder)} \) in the objective function of \eqref{eq:inactive-lambda-problem} is dominated eventually by its linearly growing component \( \regulizer \inpnorm{\separator{t}} + \; \summ{i = 1}{\dsize} \inpnorm{ \frac{1}{T} \summ{t = 1}{T} \hvar{t}'(i) \separator{t} } \). Consequently, as \( \inpnorm{\separator{t}} \) grows arbitrarily large, the objective function takes negative values with arbitrarily large magnitude. Thus, the objective function of \eqref{eq:inactive-lambda-problem} is coercive. However, due to the strict inequality of the constraint in \eqref{eq:inactive-lambda-problem}, the feasible set is open, and existence of optimal solution in such a setting is not readily available. 

If we were to relax the strict inequality constraint in \eqref{eq:inactive-lambda-problem} and instead consider the maximization problem
\begin{equation}
\label{eq:active-lambda-problem}
\begin{cases}
\begin{aligned}
& \sup_{(\separator{t})_t} \; 
&& \frac{1}{\horizon} \summ{t = 1}{\horizon} 
\begin{cases}
\begin{aligned}
r(\horder) \big( \inprod{\separator{t}}{\sample{t}} & - \error_t \inpnorm{\separator{t}} \big)^{q(\horder)}  \\
& - \; \regulizer \inpnorm{\separator{t}} - \; \summ{i = 1}{\dsize} \inpnorm{ \frac{1}{T} \summ{t = 1}{T} \hvar{t}'(i) \separator{t} }
\end{aligned}
\end{cases}
\\
& \sbjto && 
\inprod{\separator{t}}{\sample{t}} - \error_{t} \inpnorm{\separator{t}} \geq  0 \quad \text{for all } t = 1,2,\ldots, \horizon .
\end{aligned}
\end{cases}
\end{equation}
It is immediately evident that \eqref{eq:active-lambda-problem} is a problem of maximizing a coercive objective function over a closed set, and therefore, from Wierstrauss's extreme value theorem we conclude that \eqref{eq:active-lambda-problem} admits an optimal solution \( (\separator{t}')_t \). 

We show that the optimal solution \( (\separator{t}')_t \) to \eqref{eq:active-lambda-problem} satisfies the inequality constraint strictly, whereby, it is also feasible in \eqref{eq:inactive-lambda-problem}. Consequently, the maximization problem \eqref{eq:inactive-lambda-problem} admits a solution, and indeed \( (\separator{t}')_t \) is one such solution. We show the feasibility of \( (\separator{t}')_t \) in \eqref{eq:inactive-lambda-problem} by contradiction. 

For \( \alpha \geq 0 \), consider a collection \( (\separator{t}(\alpha))_t \) defined by \( \separator{s}(\alpha) \define \separator{s}' + \alpha \sample{s} \) and \( \separator{t}(\alpha) \define \separator{t}' \) for all \( t \neq s \). From the triangle inequality it immediately follows that
\begin{equation}
\label{eq:triangle-inequality}
\begin{aligned}
\inpnorm{\separator{s}(\alpha)} \; & \leq \; \inpnorm{\separator{s}'} + \alpha \inpnorm{\sample{s}} \\
\summ{i = 1}{\dsize} \inpnorm{ \frac{1}{T} \summ{t = 1}{T} \hvar{t}'(i) \separator{t}(\alpha) } \; & \leq \; \summ{i = 1}{\dsize} \inpnorm{ \frac{1}{T} \summ{t = 1}{T} \hvar{t}'(i) \separator{t}'} + \alpha \inpnorm{\sample{s}} \summ{i = 1}{\dsize} \abs{\hvar{s}(i)} 
\end{aligned}
\end{equation}

Let \( V' \), \( V(\alpha) \) denote the value of the objective function evaluated at \( (\separator{t}')_t \) and \( (\separator{t}(\alpha))_t \) respectively. Suppose that \( \inprod{\separator{s}'}{\sample{s}} - \error_s \inpnorm{\separator{s}'} = 0 \) for some \( s \in \{1,2,\ldots,T\} \), then we have
\begin{equation}
\label{eq:dummy-V-value}
\begin{aligned}
V' = \frac{1}{T} \summ{t \neq s}{} \Big( r(\horder) \big( \inprod{\separator{t}'}{\sample{t}} & - \error_t \inpnorm{\separator{t}'} \big)^{q(\horder)} - \regulizer \inpnorm{\separator{t}'} \Big) \\
&- \regulizer \inpnorm{\separator{s}'} \; - \; \summ{i = 1}{\dsize} \inpnorm{ \frac{1}{T} \summ{t = 1}{T} \hvar{t}'(i) \separator{t}'} .
\end{aligned}
\end{equation}
Collecting \eqref{eq:triangle-inequality} and \eqref{eq:dummy-V-value}, it is easily verified that
\[
\begin{aligned}
V(\alpha) - V' \ \geq \ \alpha^{q(\horder)} r(\horder) \Big( \inpnorm{\sample{s}} \big( \inpnorm{\sample{s}} - \error_s \big) \Big)^{q(\horder)} - \alpha \inpnorm{\sample{s}} \Big( \regulizer + \summ{i = 1}{\dsize} \abs{\hvar{s}(i)} \Big)
\end{aligned}
\]
Combining the facts that the quantities \( a' \define  r(\horder) \Big( \inpnorm{\sample{s}} \big( \inpnorm{\sample{s}} - \error_s \big) \Big)^{q(\horder)} \) and \( b' \define \inpnorm{\sample{s}} \Big( \regulizer + \summ{i = 1}{\dsize} \abs{\hvar{s}(i)} \Big) \) are both positive, and the maximization problem \( \max\limits_{\alpha \geq 0} \; a \alpha^q - b \alpha \) admits an optimal solution for every \( a,b > 0 \) and \( q \in ]0,1[ \) with a positive optimal value. We conclude that \( 0 <  \max\limits_{\alpha \geq 0} \; a' \alpha^{q(\horder)} - b' \alpha \) and the maximum is achieved at some \( \alpha' > 0 \).\footnote{The key idea here is that the sub-linear term \( \alpha^q \) grows faster and takes values more than the linear term \( \alpha \) as \( \alpha \) increases from \( 0 \).} Therefore, we have \( V(\alpha') > V' \). Moreover, since \( \inprod{\separator{s}(\alpha')}{\sample{s}} - \error_s \inpnorm{ \separator{s}(\alpha') } \; \geq \; \alpha' \inpnorm{\sample{s}} \big( \inpnorm{\sample{s}} - \error_s \big) \; > \; 0 \) the collection \( (\separator{t}(\alpha'))_t \) is also feasible in \eqref{eq:active-lambda-problem}, which contradicts the optimality of \( (\separator{t}')_t \) in the problem \eqref{eq:active-lambda-problem}. Therefore, \( \inprod{\separator{t}'}{\sample{t}} - \error_t \inpnorm{\separator{t}'} > 0 \) for every \( t = 1,2,\ldots,T \), and consequently, \( (\separator{t}')_t \) is also an optimal solution to \eqref{eq:inactive-lambda-problem}. Uniqueness follows easily from the strong concavity of the objective function in \eqref{eq:inactive-lambda-problem} since \( \regulizer > 0 \). 

It remains to be shown that the dictionary \( \dictionary' \define \dictionary ((\separator{t}')_t) \) computed using \eqref{eq:optimal-D-for-lambda} along with \( (\separator{t}')_t \) is a saddle point to the min-max problem \eqref{eq:DL-D-update-min-max}. From \eqref{eq:optimal-D-for-lambda}, it easily follows that 
\begin{equation}
\label{eq:optimality-of-D-in-saddle-point}
\dictionary ((\separator{t}')_t) \; = \; \argmin_{\dictionary \; \in \; \unitdictionaryset} \ 
\frac{1}{\horizon} \summ{t = 1}{\horizon} 
\begin{cases}
\begin{aligned}
r(\horder) \Big( \inprod{\separator{t}'}{\sample{t}} & - \error_t \inpnorm{\separator{t}'} \Big)^{q(\horder)} \\
& - \Big(  \regulizer \inpnorm{\separator{t}'} + \inprod{\separator{t}'}{\dictionary \hvar{t}'} \Big)
\end{aligned}
\end{cases}
\end{equation}

Since \( (\separator{t}')_t \) is an optimal solution to \eqref{eq:inactive-lambda-problem}, it must satisfy the first order optimality conditions. Using Danskin's theorem and making use of \eqref{eq:optimal-D-for-lambda} to compute the gradients of \( \summ{i = 1}{\dsize} \inpnorm{ \frac{1}{T} \summ{t = 1}{T} \hvar{t}'(i) \separator{t} } \) w.r.t. \( \separator{t} \), the optimality conditions are written
\begin{equation}
\label{eq:optimality-condition-of-lambda-in-saddle-point}
0 \; = \; \frac{r(\horder) q(\horder) \big( \sample{t} - (\error_t / \inpnorm{\separator{t}'}) \separator{t}' \big) }{ \big( \inprod{\separator{t}'}{\sample{t}} - \error_t \inpnorm{\separator{t}'} \big)^{1 - q(\horder)} } - \frac{\regulizer}{\inpnorm{\separator{t}'}} \separator{t}' - \dictionary' \hvar{t}' \; ,
\end{equation}
for every \( t = 1,2,\ldots,T \). The first order conditions \eqref{eq:optimality-condition-of-lambda-in-saddle-point} immediately imply
\begin{equation}
\label{eq:optimality-of-lambda-in-saddle-point}
(\separator{t}')_t \; \in \;
\begin{cases}
\begin{aligned}
& \argmax_{(\separator{t})_t} \; 
&& \frac{1}{\horizon} \summ{t = 1}{\horizon} 
\begin{cases}
\begin{aligned}
r(\horder) \Big( \inprod{\separator{t}}{\sample{t}} & - \error_t \inpnorm{\separator{t}} \Big)^{q(\horder)} \\
 & - \Big( \regulizer \inpnorm{\separator{t}} + \inprod{\separator{t}}{\dictionary' \hvar{t}} \Big)
\end{aligned}
\end{cases}
\\
& \sbjto && 
\inprod{\separator{t}}{\sample{t}} - \error_{t} \inpnorm{\separator{t}} >  0 \quad \text{for all } t = 1,2,\ldots, \horizon ,
\end{aligned}
\end{cases}
\end{equation}
because, \eqref{eq:optimality-of-lambda-in-saddle-point} is a maximization of a concave function, and first order conditions are sufficient as well. Collecting \eqref{eq:optimality-of-D-in-saddle-point} and \eqref{eq:optimality-of-lambda-in-saddle-point}, we conclude the proposition.
\end{proof}

\begin{proof}[Proof of Proposition \ref{proposition:optimality-condition}]

\textsf{Necessary conditions} :
If \( \dictionary\opt \) is an optimal dictionary to the dictionary learning problem \eqref{eq:DL-samples} and for each \( t = 1,2,\ldots,T \), since \( \separatorset{\dictionary\opt}{t} \neq \emptyset \), we conclude form \cite{sheriff2019LIP} that every pair \( (\hvar{t}' , \separator{t}') \in H_t (\dictionary\opt) \times \samplecost{\dictionary\opt}{t} \cdot \separatorset{\dictionary\opt}{t} \) is a saddle point to the equivalent min-max form \eqref{eq:coding-problem-inf-sup-equivalent} of the encoding problem. Consequently, we have
\begin{equation}
\label{eq:optimality-necessary-lambda}
\separator{t}' \; \in \; 
\begin{cases}
\begin{aligned}
& \argmax_{\separator{t}} \ \ && 
\begin{cases}
\begin{aligned}
r (\horder) \Big( \inprod{\separator{t}}{\sample{t}} & - \error_t \inpnorm{\separator{t}} \Big)^{q (\horder)} \\
& - \Big( \regulizer \inpnorm{\separator{t}} + \inprod{\separator{t}}{\dictionary\opt \hvar{t}'} \Big)
\end{aligned}
\end{cases}
\\
& \sbjto &&  \inprod{\separator{t}}{\sample{t}} - \error_t \inpnorm{\separator{t}} > 0 \; .
\end{aligned}
\end{cases}
\end{equation}
Using the definition \eqref{eq:J-map-def} of \( J_t (\cdot , \cdot) \) and the inclusion \eqref{eq:optimality-necessary-lambda} for each \( t = 1,2,\ldots,T \), we conclude from Danskin's theorem that \( \frac{1}{T} \summ{t = 1}{T} ( \separator{t}' {\hvar{t}'}\transp ) \) is a sub-gradient of the mapping \( \unitdictionaryset \ni \dictionary \longmapsto \frac{1}{T} \summ{t = 1}{T} J_t(\dictionary , \hvar{t}') \), for every collection \( (\separator{t}')_t \) satisfying \( \separator{t}' \in \samplecost{\dictionary\opt}{t} \cdot \separatorset{\dictionary\opt}{t} \) for each \( t \).

From the optimality of \( \dictionary\opt \), we know from the proof of Lemma \ref{lemma:equivalent-DL} that the following inclusion holds
\begin{equation}
\label{eq:necessary-inclusion-D}
\dictionary\opt \; \in \; \argmin_{\dictionary \; \in \; \unitdictionaryset} \ \frac{1}{T} \summ{t = 1}{T} J_t(\dictionary , \hvar{t}') \; . 
\end{equation}
Applying first order necessary optimality conditions to \eqref{eq:necessary-inclusion-D} we conclude that there exists \( \separator{t}' \in \samplecost{\dictionary\opt}{t} \cdot \separatorset{\dictionary\opt}{t} \) for every \( t = 1,2,\ldots,T \), such that the necessary condition
\[
\dictionary\opt \; \in \; \argmax_{\dictionary \; \in \; \unitdictionaryset} \; \trace \left( \dictionary \cdot \frac{1}{T} \summ{t = 1}{T} \samplecost{\dictionary\opt}{t} \big( \separator{t}' {\hvar{t}'}\transp \big) \right) \text{ follows.}
\]

\textsf{Sufficient conditions}
If \eqref{eq:optimality-condition} holds for a dictionary \( \dictionary' \) and a collection \( (\hvar{t}' , \separator{t}')_t \) satisfying \( ( \hvar{t}' , \separator{t}' ) \in H_t (\dictionary') \; \times \; \samplecost{\dictionary\opt}{t} \cdot \separatorset{\dictionary'}{t} \) for every \( t = 1,2,\ldots,T \). On the one hand we see that the inclusion
\begin{equation}
\label{eq:sufficient-conditions-inclusion-h}
(\hvar{t}')_t \in \argmin_{(\hvar{t})_t \subset \costatomset} \; \frac{1}{T} \summ{t = 1}{T} J_t (\dictionary' , \hvar{t}) \; ,
\end{equation}
follows straightforward from the fact that \( H_t (\dictionary') = \argmin\limits_{\hvar{t} \; \in \; \costatomset} \; J_t (\dictionary',\hvar{t}) \) for every \( t = 1,2,\ldots,T \). 

On the other hand, observe that the optimization problem 
\begin{equation}
\label{eq:sufficient-conditions-inclusion-D-1}
\min_{\dictionary \in \unitdictionaryset} \  \frac{1}{T} \summ{t = 1}{T} J_t (\dictionary , \hvar{t}') ,    
\end{equation}
is convex for which first order conditions are sufficient as well. Since \( \frac{1}{T} \summ{t = 1}{T} ( \separator{t}' {\hvar{t}'}\transp ) \) is a sub-gradient of the map \( \unitdictionaryset \ni \dictionary \longmapsto \frac{1}{T} \summ{t = 1}{T} J_t(\dictionary , \hvar{t}') \), for any collection \( (\separator{t}')_t \) satisfying \( \separator{t}' \in \samplecost{\dictionary\opt}{t} \cdot \separatorset{\dictionary\opt}{t} \) for each \( t \). The condition \eqref{eq:optimality-condition} (which is same as the first order condition for \eqref{eq:sufficient-conditions-inclusion-D-1}) is sufficient to conclude the inclusion
\begin{equation}
\label{eq:sufficient-conditions-inclusion-D}
\dictionary' \in  \argmin_{\dictionary \in \unitdictionaryset} \; \frac{1}{T} \summ{t = 1}{T} J_t (\dictionary , \hvar{t}') .
\end{equation}

Collecting \eqref{eq:sufficient-conditions-inclusion-h} and \eqref{eq:sufficient-conditions-inclusion-D}, stationarity of \( \dictionary' \) to \eqref{eq:DL-samples} follows immediately from the convexity of \( J_t (\cdot , \cdot) \) individually in both the arguments. The proof is now complete.
\end{proof}

\vskip 0.2in
\bibliographystyle{alpha}
\bibliography{ref}

\end{document}